\documentclass{article}

\usepackage{arxiv}

\usepackage[pdftex, pdfstartview={FitV}, pdfpagelayout={TwoColumnLeft},bookmarksopen=true,plainpages = false, colorlinks=true, linkcolor=black, citecolor = black, urlcolor = black,filecolor=black , pagebackref=false,hypertexnames=false, plainpages=false, pdfpagelabels ]{hyperref}

\usepackage[authormarkuptext=name,addedmarkup=bf,authormarkupposition=right]{changes}
\definechangesauthor[name={BTL}, color={blue}]{bl}

\usepackage{setspace}
\usepackage{graphicx}
\usepackage{epstopdf}
\usepackage{amssymb,amsmath}

\usepackage{amsthm}
\usepackage{mathtools}
\usepackage{cuted}
\usepackage[font=footnotesize]{subcaption}

\usepackage[font=footnotesize]{caption}
\usepackage{xcolor}
\usepackage{units}
\usepackage{algorithm}
\usepackage[noend]{algpseudocode}
\usepackage{balance}
\usepackage[sort,compress,noadjust]{cite}
\usepackage{tabularx}
\usepackage{nameref}
\usepackage{centernot}

\usepackage{enumitem}

\theoremstyle{plain}
\newtheorem{theorem}{Theorem}
\theoremstyle{plain}
\newtheorem{proposition}{Proposition}
\theoremstyle{plain}
\newtheorem{lemma}{Lemma}
\theoremstyle{plain}

\theoremstyle{definition}

\theoremstyle{definition}

\theoremstyle{remark}
\newtheorem{remark}{Remark}

\usepackage[hang,flushmargin]{footmisc}

\usepackage[capitalize]{cleveref}
\crefformat{equation}{(#2#1#3)}
\Crefformat{equation}{Equation~(#2#1#3)}
\Crefname{equation}{Equation}{Eqs.}

\usepackage{accents}

\newcommand{\pre}[2]{\prescript{#1}{}{#2}}

\title{Quaternion Sliding Variables in Manipulator Control}

\author{
 Brett T. Lopez \\
  VECTR Laboratory \\
  University of California, Los Angeles\\
  Los Angeles, CA 90095 \\
  \texttt{btlopez@ucla.edu} \\
   \And
 Jean-Jacques Slotine \\
  Nonlinear Systems Laboratory\\
  Massachusetts Institute of Technology \\
  Cambridge, MA 02139 \\
  \texttt{jjs@mit.edu} \\
}

\begin{document}
\maketitle
% \begin{abstract}

% \end{abstract}

% keywords can be removed
%\keywords{First keyword \and Second keyword \and More}

\section{Introduction}
In this note, we present two quaternion-based sliding variables for controlling the orientation of a manipulator's end-effector. 
Both sliding variables are free of singularities and represent global exponentially convergent error dynamics that do not exhibit unwinding when used in feedback.
The choice of sliding variable is dictated by whether the end-effector's angular velocity vector is expressed in a local or global frame, and is a matter of convenience.
Using quaternions allows the end-effector to move in its full operational envelope, which is not possible with other representations, e.g., Euler angles, that introduce representation-specific singularities.
% In other words, in these representations, singularities arise because of the choice of orientation representation rather than being intrinsic to the system. 
Further, the presented stability results are global rather than almost global, where the latter is often the best one can achieve when using rotation matrices to represent orientation \cite{koditschek1989application,bhat2000topological,cortes2019sliding}.
Quaternion sliding variables have been proposed before, but those found in the literature \cite{lo1995smooth,jan2004minimum,yeh2010sliding,sanchez2013time,zou2017nonlinear} either improperly define the error quaternion or angular velocity vector, are unable to prevent the unwinding phenomenon \cite{mayhew2011quaternion}, or only achieve asymptotic convergence.
Quaternion-based end-effector control strategies also exist \cite{yuan1988closed,caccavale1998resolved,chiaverini1999unit,xian2004task}, but at best achieve asymptotic convergence, and none address unwinding. 

\section{Preliminaries}
Consider the dynamics of a 6-link\footnote{The results presented here are also applicable to redundant manipulators where degrees of freedom exceeds 6.} robot manipulator with revolute joints
% Consider the dynamics of a 6-link robot manipulator with revolute joints
\begin{equation}
\label{eq:robot_dyn}
    H(\theta) \ddot{\theta} + C(\theta,\dot{\theta}) \dot{\theta} + g(\theta) = \tau,
\end{equation}
where $\theta \in \mathbb{R}^6$ are the joint angles (or joint displacements)\footnote{We adopt the notation $\theta$ for joint angles to avoid ambiguity in representing quaternions with $q$.}, $\tau \in \mathbb{R}^6$ are the input torques, $H(\theta) \in \mathbb{R}^{6\times 6}$ is the inertia tensor, $C(\theta,\dot{\theta}) \in \mathbb{R}^{6\times 6}$ is the centripetal-Coriolis torque tensor,  and $g(\theta)\in \mathbb{R}^6$ are the torques due to gravity. 
The inertia and centripetal-Coriolis torque tensor have a skew-symmetric property $\dot{H}(\theta) - 2\,C(\theta,\dot{\theta}) = 0$.
In the case where desired joint angles $\theta_d$ are known, one can employ several control techniques \cite{slotine1991applied} so $\theta$ tracks a time-varying reference $\theta_d(t)$.
However, in many applications, e.g., pick-and-place or catching, it is more convenient to express a desired position and orientation of the end-effector rather than joint angles. 
While computing end-effector positions given joint angles is systematic, the practical
question of computing joint angles given end-effector positions is 
more challenging~\cite{asada1986robot}. 
The situation is easier in velocity-space,
where so-called forward kinematics of a robot arm computes end-effector velocities from joint velocities via the expression
\begin{equation}
    \label{eq:forward_kin}
    \left[ \begin{array}{c} v \\ \omega \end{array} \right] = J(\theta) \, \dot{\theta},
\end{equation}
where $v \in \mathbb{R}^3$ and $\omega \in \mathbb{R}^3$  are the translational and angular velocities, respectively. 
The square matrix $J(\theta) \in \mathbb{R}^{6\times 6}$ is commonly referred to as the \emph{geometric Jacobian} and is known to contain singularities that are intrinsic to the robot arm. 
Extensive results are available on determining, avoiding, or otherwise addressing such singularities in manipulators~\cite{asada1986robot}.
In the sequel, we will assume that singularities of $J(\theta)$ are avoided, with the understanding that the aforementioned methods can be simply combined with the proposed approach.

The goal of this work is to construct a control law that computes joint torques $\tau$ so the end-effector's position and orientation track a time-varying trajectory.
Let $\pre{\mathcal{I}}p \in \mathbb{R}^3$ denote the position (not necessarily in Cartesian coordinates) of the end-effector in an inertial frame $\mathcal{I}$.
Further, let the unit quaternion $\prescript{\mathcal{I}}{\mathcal{B}}{q} \in \mathbb{S}^3$ be the orientation of a reference frame $\mathcal{B}$ (attached to the end-effector) with respect to the inertial frame $\mathcal{I}$.
The unit quaternion $\prescript{\mathcal{I}}{\mathcal{B}}{q}$ is a four vector of unit length with a real and vector part, i.e., $\prescript{\mathcal{I}}{\mathcal{B}}{q} = (\,q^\circ,\vec{q}\,\,)$.
Quaternions form a Lie group under the non-commutative multiplication operator $\otimes$ where the product of two quaternions $q_1$ and $q_2$ is
\begin{equation*}
    q_1 \otimes q_2 =  ( \, q_1^\circ q_2^\circ - \vec{q}_1^\top \vec{q}_2, \,  q_1^\circ\vec{q}_2 + q_2^\circ\vec{q}_1 + \vec{q}_1 \times \vec{q}_2\, ).
\end{equation*}
The inverse of quaternion $q$ is its conjugate $\, q^* = (\,q^\circ,-\vec{q}\,\,)\, $ where $q \otimes q^* = q^* \otimes q = (\, \pm 1,\vec{0}\, )$.
Quaternions double cover the special orthogonal group $\mathbb{SO}(3)$ so $q$ and $-q$ represent the same orientation.
This can lead to the so-called unwinding phenomenon where longer-than-necessary maneuvers occur if the feedback controller is not carefully designed.
More specifically, unwinding is the phenomenon where the manipulator will exhibit an unstable / inefficient motion even though it is arbitrarily close to the desired orientation.

Quaternion kinematics depends on whether the angular velocity vector is defined in a local or global frame.
The frame of reference can be chosen based on sensing or application.
For instance, it might be easier to directly measure angular velocity in the local, i.e., end-effector frame, with a rigidly mounted inertial measurement unit.
Otherwise, angular velocity can be computed using the forward kinematics in \cref{eq:forward_kin}.
Note that \cref{eq:forward_kin} holds whether a local or global frame are used since the two are related via $q$.
If ${\omega}$ is the angular velocity in the local frame, then the quaternion kinematics are
\begin{equation*}
    \prescript{\mathcal{I}}{\mathcal{B}}{\dot{q}} = \tfrac{1}{2} \prescript{\mathcal{I}}{\mathcal{B}}{q} \otimes \left(\, 0,\,\omega \,\right).
    % \hspace{0.5cm} \prescript{\mathcal{I}}{\mathcal{B}}{\dot{q}} = \tfrac{1}{2}  \left(\, 0,\, {\omega}^{\mathcal{I}}\,\right) \otimes \prescript{\mathcal{I}}{\mathcal{B}}{q},
\end{equation*}
Conversely, if one prefers to represent the angular velocity vector in the inertial frame---denoted as $\pre{\mathcal{I}}{\omega}$ where $\pre{\mathcal{I}}{\omega} = \prescript{\mathcal{I}}{\mathcal{B}}{R} \, \omega$ with $\prescript{\mathcal{I}}{\mathcal{B}}{R} \in \mathbb{SO}(3)$ being the rotation matrix formed from $\prescript{\mathcal{I}}{\mathcal{B}}{q}$---then the quaternion kinematics are 
\begin{equation*}
    \prescript{\mathcal{I}}{\mathcal{B}}{\dot{q}} = \tfrac{1}{2}  \left(\, 0,\, \pre{\mathcal{I}}\omega\,\right) \otimes \prescript{\mathcal{I}}{\mathcal{B}}{q} = \tfrac{1}{2}  \left(\, 0,\, \prescript{\mathcal{I}}{\mathcal{B}}{R}\,\omega\,\right) \otimes \prescript{\mathcal{I}}{\mathcal{B}}{q}.
\end{equation*}
In the sequel we will let $q$ and $R$ be the orientation expressed as a quaternion or rotation matrix (respectively) of the end effector with respect to the inertial frame $\mathcal{I}$.

Given a desired (possibly time-varying) position $\,p_d(t)\,$ and orientation $\,q_d(t)\,$ trajectory, the challenge becomes defining tracking errors $\,p_e\,$ and $\,q_e\,$ that can be used to compute joint toques $\,\tau$.
This note shows how carefully designing quaternion-based sliding variables yields global exponential convergence to the desired orientation of the end effector.
% Moreover, the proposed quaternion sliding variable can also be employed for position control when it is more convenient to work in cylindrical or spherical coordinates.

\section{Main Results}
\label{sec:sliding}

\subsection{Overview}
Control in velocity-space is naturally formulated using so-called sliding variables.
Conceptually, sliding variables represent an exponentially stable hierarchy where convergence to a reference velocity implies convergence to a desired position trajectory. 
Therefore, if the manifold corresponding to the sliding variable being zero is made invariant through control then it is trivial to show that the system converges to the desired position exponentially. 
However, defining sliding variables for systems that evolve on a non-Euclidean manifold is non-trivial because one must define a suitable tracking error and show the corresponding error dynamics converge exponentially.
This section will present and analyze two quaternion-based sliding variables that represent global exponentially convergent error dynamics.
% The choice of sliding variable is dictated by the frame angular velocity vector being expressed in local or global frame.
The sliding variable with the angular velocity vector expressed in a local frame was originally proposed in \cite{lopez2021sliding} for attitude control for aerial vehicles. 
In the sequel, we define the error quaternion as $q_e \triangleq q_d^* \otimes q$ where $q_d^*$ is the conjugate of the desired orientation and $q$ is the current orientation of the end-effector.

% \subsection{Error Quaternion Kinematics}
% Consider the error quaternion $q_e \triangleq q_d^* \otimes q$ where $q_d^*$ is the conjugate of the desired orientation and $q$ is the current orientation of the end-effector.
% In order to achieve exponentially convergent dynamics for the error quaternion $q_e$, one must appropriately define the angular velocity error $\omega_e$.
% It will be seen that the choice of local or global coordinates will yield different definitions for $\omega_e$.
% For instance, consider the local coordinate case. 
% If 
% \begin{equation}
%     \label{eq:ang_v_local}
%     {\omega_e} \triangleq {\omega} - R(q_e)^\top {\omega_d},
% \end{equation}
% then the kinematics for $q_e$ are simply \bxx{finish}
% Conversely, for the global coordinate case one can show 
% \begin{equation}
%     \label{eq:ang_v_global}
%     {\omega_e} \triangleq R(q_d)^\top[{\omega} - {\omega_d}],
% \end{equation}

% Differentiating the definition of $q_e$ yields
% \begin{equation*}
%     \dot{q}_e = \tfrac{1}{2} q_e \otimes \left[\begin{array}{c} 0 \\ \omega_e\end{array} \right]
% \end{equation*}
% where ${\omega_e} \triangleq {\omega} - R(q_e)^\top {\omega}$.

\subsection{Sliding Variables in $\mathbb{S}^3\times \mathbb{R}^3$ with Local Angular Velocity}
First consider the case where the angular velocity vector is expressed in a local frame.
The sliding variable definition from \cite{lopez2021sliding} can be used to obtain exponentially convergent error dynamics. 
Let the local angular velocity error be defined as ${\omega_e} \triangleq {\omega} -  R_e^\top {\omega_d}$ where $R_e$ is the error rotation matrix formed from $q_e$. 
Conceptually, $R_e$ is present because ${\omega_d}$ and ${\omega}$ are tangent vectors on the three sphere at \emph{different} locations. 
As a result, ${\omega_d}$ must be transformed into the same frame as ${\omega}$ for the vector difference to be meaningful.
The error quaternion has the kinematics 
\begin{equation}
    \dot{q}_e = \tfrac{1}{2} q_e \otimes (\, 0, \, {\omega_e} \, ) =  \tfrac{1}{2} ( \, -\vec{q}_e^\top {\omega_e}, ~ q_e^\circ {\omega_e} + \vec{q}_e \times {\omega_e} \, ).
    \label{eq:qe_dynamics}
\end{equation}
Let the quaternion sliding variable ${s}_q$ be defined as
\begin{equation}
    {s}_q \triangleq {\omega_e} + 2\, \lambda  \, \mathrm{sgn}\left(q_e^\circ\right)\vec{q}_e,
    \label{eq:s_local}
\end{equation}
with $\lambda >0$ and $\mathrm{sgn}(x) = x / |x|$ for $x \neq 0$ and $\mathrm{sgn}(0) = 1$.
Observe $\mathrm{sgn}$ differs from its standard definition, but the modification will be critical for establishing global exponential convergence.
\cref{proposition:qe_local} shows that if ${s}_q = 0$ indefinitely via feedback then $\|\vec{q}_e\| \rightarrow 0$ exponentially at rate $\lambda$, which is equivalent to $q \rightarrow \pm q_d$ exponentially.

\begin{proposition}
    Assume the manifold ${\mathcal{S}} \triangleq \{({q}_e,\omega_e) \in \mathbb{S}^3 \times \mathbb{R}^3 \, | \, {s_q({q}_e,\omega_e)}=0\}$ is made invariant. Then, for any trajectory initialized on ${\mathcal{S}}$, the vector part of $q_e$ globally exponentially converges to zero with rate $\lambda$.
    \label{proposition:qe_local}
\end{proposition}

\begin{proof}
    Consider the Lyapunov function $V = \|\vec{q}_e\|^2$.
    Differentiating along the quaternion error dynamics \cref{eq:global_err} yields,
    \begin{equation}
        \dot{V} = 2 \vec{q}_e^\top\dot{\vec{q}}_e = \vec{q}_e^\top \left[ q_e^\circ \omega_{e} + \omega_{e} \times \vec{q}_e \right] = q_e^\circ \vec{q}_e^\top \omega_{e},
        \label{eq:v_dot_local}
    \end{equation}
    where the third equality is obtained by noting the cross product between $\vec{q}_e$ and $\omega_{e}$ produces a vector perpendicular to $\vec{q}_e$.
    Since the manifold ${\mathcal{S}}$ is invariant, i.e., ${s}_q=0$ indefinitely by assumption, then, from \cref{eq:s_local}, ${\omega_e} = -2 \, \lambda  \, \mathrm{sgn}\left(q_e^\circ\right) \, \vec{q}_e$ so \cref{eq:v_dot_local} becomes
    \begin{equation}
        \label{eq:v_dot_qe_local}
        \dot{V} = \tfrac{d}{dt} \|\vec{q}_e\|^2 = - 2 \, \lambda \, |q_e^\circ| \,  \|\vec{q}_e\|^2 = - 2 \, \lambda \,  \|\vec{q}_e\|^2 \, \sqrt{1-\|\vec{q}_e\|^2},
    \end{equation}
    where we have made use of the fact that $q_e$ is a unit quaternion. 
    To show exponential convergence, we can leverage \cref{lemma:ode} by letting $x = \|\vec{q}_e\|^2$ and $\sigma = 2 \, \lambda $ to obtain $\|\vec{q}_e(t)\| \leq 2 \, e^{-\lambda t}$ for all $\|\vec{q}_e\| \neq 1$ where we have used a looser bound than that in \cref{lemma:ode} by noting $c \leq 1$.
    Hence $\lim_{t \rightarrow \infty} \|\vec{q}_e(t)\| = 0$ exponentially at rate $\lambda$.
    To establish global convergence, the radially unboundedness criterion cannot be used since $q_e$ lives on a compact manifold.
    It is sufficient to show that $\|\vec{q}_e\| = 1$, i.e., when $q_e^\circ = 0$, is not an equilibrium point. 
    The closed-loop kinematic equation for the real part of the error quaternion is $\dot{q}_e^\circ =  \lambda \, \mathrm{sgn}(q_e^\circ) \, \| \vec{q}_e \|^2$ which is non-zero when $q_e^\circ = 0$, i.e., when $\|\vec{q}_e\| = 1$, by the modified definition of $\mathrm{sgn}$. 
    Therefore, for any trajectory initialized on the invariant manifold $\mathcal{S}$ then $\|\vec{q}_e\|$ globally exponentially converges to zero as desired.
\end{proof}

\subsection{Sliding Variables in $\mathbb{S}^3\times \mathbb{R}^3$ with Global Angular Velocity}
Now consider the case where the angular velocity vector is expressed in a global frame.
Since $q_e = q_d^* \otimes q$ as before then error quaternion kinematics are
\begin{equation}
\label{eq:global_err}
    \dot{q}_e = \tfrac{1}{2} (\, 0, \, \tilde{\omega}\,) \otimes q_e = \tfrac{1}{2}(\,-\tilde{\omega}^\top \vec{q}_e, ~ q_e^\circ \tilde{\omega} + \tilde{\omega} \times \vec{q}_e \, ), 
\end{equation}
with $ \tilde{\omega} \triangleq R_d^\top ({\omega} - {\omega_d})$ where $R_d$ is the rotation matrix formed via the desired quaternion $q_d$.
Letting ${\omega_e} \triangleq {\omega} - {\omega_d}$, consider the quaternion sliding variable
\begin{equation}
\label{eq:s_global}
    \mathfrak{s}_q \triangleq {\omega_e} + 2 \, \lambda \,  \mathrm{sgn}(q_e^\circ) \, R_d \, \vec{q}_e
\end{equation}
with $\lambda >0$ and $\mathrm{sgn}$ defined as before. 
\cref{proposition:qe_global} shows that if $\mathfrak{s}_q = 0$ indefinitely via feedback then $\| \vec{q}_e \| \rightarrow 0$ exponentially at rate $\lambda$.
\begin{proposition}
    Assume the manifold ${\mathfrak{S}} \triangleq \{({q}_e,\omega_e) \in \mathbb{S}^3 \times \mathbb{R}^3 \, | \, {\mathfrak{s}_q({q}_e,\omega_e)}=0\}$ is made invariant. Then, for any trajectory initialized on ${\mathfrak{S}}$, the vector part of $q_e$ globally exponentially converges to zero with rate $\lambda$.
    \label{proposition:qe_global}
\end{proposition}

\begin{proof}
Following nearly identical steps as \cref{proposition:qe_local}, one arrives to the same differential equations for $\|\vec{q}_e\|^2$ given by \cref{{eq:v_dot_qe_local}}.
\cref{lemma:ode} can then be applied to establish exponential stability of the origin. 
The result can be strengthened to global as done in \cref{proposition:qe_local}.
The only minor point worth stating is the proof leverages the orthogonality property of $R_d$, i.e., $R_d R_d^\top = R_d^\top R_d = I$.
\end{proof}

\subsection{Inverse Kinematics Torque Controller}
This subsection will show how the quaternion sliding variables defined above can be used in task-space control to track a desired orientation trajectory of the end-effector.
Several controllers for \cref{eq:robot_dyn} have been proposed in the literature if the desired joint angles $\theta_d$ are specified. 
The well-known Slotine-Li sliding controller is \cite{slotine1991applied}
\begin{equation}
    \label{eq:sliding_control}
    \tau = H(\theta) \ddot{\theta}_r + C(\theta,\dot{\theta}) \dot{\theta}_r + g(\theta) - K s_\theta,
\end{equation}
where $\, \dot{\theta}_r = \dot{\theta}_d - \lambda \, (\theta - \theta_d) \, $ and $\, s_\theta = \dot{\theta} - \dot{\theta}_r$.
It is straightforward to show that this choice of $\tau$ yields $s_\theta \rightarrow 0$ exponentially.
Moreover, since $\, s_\theta\, $ forms a hierarchy with $\, \theta_e = \theta - \theta_d \, $, then $\, \theta \rightarrow \theta_d \, $ exponentially with rate $\, \lambda$.
The primary limitation of \cref{eq:sliding_control} within the context of task-space control is the explicit dependency on $\theta_d$. 
Knowing $\theta_d$ can be relaxed by instead using inverse kinematics and appropriately defining sliding variables for the end-effector position and orientation \cite{niemeyer1991performance}. 
In conjunction with the quaternion sliding variable, we will employ the position sliding variable
\begin{equation}
\label{eq:s_p}
    \begin{aligned}
        s_p & = \dot{p}_e + \sigma \, p_e  = v_e + \sigma \, p_e
    \end{aligned}
\end{equation}
where $p_e = p - p_d$. 
Clearly, if the manifold $\mathcal{S}_p = \{ (p_e,\,v_e) \, | \, s_p(p_e,\,v_e) = 0 \}$ is made invariant via control then $p \rightarrow p_d$ exponentially at rate $\sigma$.
The following theorem shows that the end-effector will convergence to the desired position and orientation trajectory using \cref{eq:s_p,eq:s_local} in feedback.
\begin{theorem}
\label{thm:ik_quat}
Assume a smooth desired position $p_d(t)$ and orientation $q_d(t)$ trajectory is designed to avoid singularities of the geometric Jacobian $J(\theta)$.
The end-effector will globally exponentially converge to $p_d(t)$ and $q_d(t)$ with 
\begin{equation}
\label{eq:ik_control}
    \tau = H(\theta) \ddot{\theta}_r + C(\theta,\dot{\theta}) \dot{\theta}_r + g(\theta) - K J^{-1}(\theta) \left[ \begin{array}{c}
         s_p  \\
         s_q
    \end{array} \right],
\end{equation}
where $\,s_p = v_e + \sigma\, p_e$, $\,s_q = \omega_e + 2\, \lambda\, \mathrm{sgn}\left(q_e^\circ\right)\vec{q}_e$, and 
\begin{equation}
\label{eq:ik_vel}
    \dot{\theta}_r = J^{-1}(\theta) \left[ \begin{array}{c} v_d(t) - \sigma \, p_e \\ R_e^\top \omega_d(t) - 2\, \lambda\, \mathrm{sgn}\left(q_e^\circ\right)\vec{q}_e \end{array} \right].
\end{equation}
\end{theorem}
\begin{proof}
Under the premise that $\,J(\theta)\,$ is non-singular for the entire trajectory $\,p_d(t)\,$ and $\,q_d(t)$, it is straightforward to show $\, \dot{\theta} \rightarrow \dot{\theta}_r\, $ with \cref{eq:ik_control} and the Lyapunov function $\, V = \tfrac{1}{2} s_\theta^\top H(\theta) s_\theta$ \cite{slotine1991applied}. 
From the kinematic relationship between joint and end-effector velocities and the choice of $\dot{\theta}_r$, one can show
\begin{equation*}
    \begin{aligned}
        J(\theta)(\dot{\theta} - \dot{\theta}_r)  & = \left[ \begin{array}{c}
             v - v_d + \sigma\, p_e  \\
             \omega - R_e^\top \omega_d + \lambda\, \mathrm{sgn}\left(q_e^\circ\right)\vec{q}_e 
        \end{array} \right] =  \left[ \begin{array}{c}
             s_p \\
             s_q 
        \end{array} \right].
    \end{aligned}
\end{equation*}
Since $\, \dot{\theta} \rightarrow \dot{\theta}_r\, $ and $J(\theta) \succ 0$ along the trajectory then necessarily $\, s_p \rightarrow 0\, $ and $\, s_q \rightarrow 0$.
Therefore, since $\,s_p\,$ and $\,s_q\,$ represent exponentially convergent error dynamics, then $\,p \rightarrow p_d(t)\,$ and $\,q \rightarrow \pm q_d(t)\,$ exponentially as desired.
\end{proof}

\begin{remark}
    If the angular velocity vector is instead expressed in a global frame then the quaternion sliding variable $\mathfrak{s}_q$ from \cref{eq:s_global} can replace $s_q$ in \cref{eq:ik_control} and the same convergence result will be obtained.
\end{remark}

\begin{remark}
    The quaternion $q$ is allowed to converge to $q_d$ or $-q_d$ as these represent the same orientation (double cover property). 
    The quaternion sliding variables ensure $q$ converges to whichever is closer.
\end{remark}

The control law derived in \cref{thm:ik_quat} is similar in structure to other task-space controllers found in the literature.
However, ours is the first to (i) achieve global exponential convergence and (ii) prevent the unwinding phenomenon for both local and global angular velocity representations.
The latter is particularly important because unwinding will cause the manipulator to exhibit an unstable / inefficient motion arbitrarily close to the desired orientation.
Within the context of safety, this phenomenon is extremely undesirable as the manipulator will behave erratically despite being near the target orientation.
Note the above results apply more generally to any controller whose desired behavior is formulated in terms of controlling velocity $\dot{\theta}$ to some reference velocity $\dot{\theta}_r$ which may itself depend, e.g., on position.

\section{Conclusion}
In this note, we presented two quaternion-based sliding variables that ensure exponential convergence of an end-effector's orientation to any time-varying orientation command that avoids singularities in the geometric Jacobian.
The choice of sliding variable is solely dictated by whether it is more convenient to express the end-effector's angular velocity in a local or global frame; regardless, both yield global exponential convergence and systematically prevent unwinding. 
Future work will investigate extending these results to the dual quaternion representation.

\section{Appendix}
\begin{lemma}
\label{lemma:ode}
    Let $x$ be restricted to the interval $x \in [0,\,1)$. The origin is exponentially stable with convergence rate $\sigma > 0$ for the nonlinear system 
    \begin{equation}
    \label{eq:sys}
        \dot{x} = - \sigma \, x \, \sqrt{1-x}.
    \end{equation}
\end{lemma}
\begin{proof}
    \Cref{eq:sys} is separable so with the substitution $ y = \sqrt{1-x}$ one obtains
    \begin{equation*}
        \int_{y_0}^{y} \frac{dz}{z^2 - 1} = - \sigma \, t,
    \end{equation*}
    which has the solution
    \begin{equation*}
        y(t) = \frac{1-c\,e^{-\sigma t}}{1 + c \, e^{-\sigma t}} ~~ \implies ~~ x(t) = \frac{4\, c \, e^{-\sigma t}}{(1+c\,e^{-\sigma t})^2} \leq  {4\, c \, e^{-\sigma t}}
    \end{equation*}
    for $c = (1-\sqrt{1-x_0}) / (1+\sqrt{1-x_0})$. 
    It is clear $x(t)$ tends to zero exponentially at rate $\sigma$ and overshoot $4\,c$ so the origin is exponentially stable.
\end{proof}

\balance
\bibliographystyle{ieeetr}
\bibliography{ref}

\end{document}